\documentclass{article}
\usepackage{spconf}

\usepackage[english]{babel}

\usepackage{ifpdf}

\usepackage{cite} 
\usepackage{url}
\usepackage{hyperref}

\usepackage[pdftex]{graphicx}
\graphicspath{{./figures/}}
\usepackage{color}
\usepackage{pgf, tikz, pgfplots}
\usetikzlibrary{shapes, arrows, automata, plotmarks}
\usetikzlibrary{calc,hobby,decorations}

\usepackage[cmex10]{amsmath}
\usepackage{amsfonts, amssymb, amsthm}
\usepackage{mathrsfs}

\usepackage{algorithm,algpseudocode}
\algnewcommand{\LeftComment}[1]{\Statex \(\triangleright\) #1}

\usepackage[margin=15mm]{geometry}
\usepackage{enumerate}
\usepackage{multirow}
\usepackage{rotating}
\usepackage{subcaption}
\input{pennColors.sty}
\usepackage{needspace}

\input{mySymbol.sty}

\def\Tr{\mathsf{T}}
\def\Hr{\mathsf{H}}

\newtheorem{theorem}{\hspace{0pt}\bf Theorem}

\newtheorem{definition}{\hspace{0pt}\bf Definition}

\title{STABILITY OF GRAPH NEURAL NETWORKS TO RELATIVE PERTURBATIONS}
\twoauthors
  {Fernando Gama, Alejandro Ribeiro\sthanks{Supported by NSF CCF 1717120, ARO W911NF1710438, ARL DCIST CRA W911NF-17-2-0181, ISTC-WAS and Intel DevCloud.}}
	{University of Pennsylvania\\
	Dept. of Electrical and Systems Engineering\\
	Philadelphia, PA}
  {Joan Bruna\sthanks{Partially supported by the Alfred P. Sloan Foundation, NSF RI-1816753, NSF CAREER CIF 1845360, and Samsung Electronics.}}
	{New York University\\
	Courant Institute of Mathematical Sciences\\
	New York, NY}
\begin{document}
\ninept
\maketitle
\begin{abstract}
Graph neural networks (GNNs), consisting of a cascade of layers applying a graph convolution followed by a pointwise nonlinearity, have become a powerful architecture to process signals supported on graphs. Graph convolutions (and thus, GNNs), rely heavily on knowledge of the graph for operation. However, in many practical cases the GSO is not known and needs to be estimated, or might change from training time to testing time. In this paper, we are set to study the effect that a change in the underlying graph topology that supports the signal has on the output of a GNN. We prove that graph convolutions with integral Lipschitz filters lead to GNNs whose output change is bounded by the size of the relative change in the topology. Furthermore, we leverage this result to show that the main reason for the success of GNNs is that they are stable architectures capable of discriminating features on high eigenvalues, which is a feat that cannot be achieved by linear graph filters (which are either stable or discriminative, but cannot be both). Finally, we comment on the use of this result to train GNNs with increased stability and run experiments on movie recommendation systems.
\end{abstract}
%


\section{Introduction} \label{sec:intro}

Networks such as power grids \cite{Owerko18-Power}, transportation networks \cite{Li18-Traffic} or weather sensor networks \cite{Isufi19-VARMA} generate data with an irregular structure dictated by the topology of the network. This data can be modeled as a graph signal by assigning each entry to a node in some underlying given graph that describes the network \cite{Ortega18-GSP}. The graph shift operator (GSO) is a linear map between graph signals where the output value at each node is a weighted average of the input values at neighboring nodes. The GSO is thus any matrix that respects the sparsity of the graph (adjacency \cite{Sandryhaila13-DSPG}, Laplacian matrix \cite{Shuman13-SPG}, etc.), and the output is said to be a \emph{shifted} version of the input.

The operation of graph convolution, defined as a linear combination of shifted version of the signal, is used to compute the output of graph filters in an efficient and decentralized fashion \cite{Segarra17-Linear, Coutino19-Distributed}. Furthermore, graph convolutions are used to build graph neural networks (GNNs), as a cascade of layers each of which applies a graph convolution, followed by a pointwise nonlinearity \cite{Bruna14-DeepSpectralNetworks, Defferrard17-CNNGraphs, Gama19-Architectures}. GNNs offer a nonlinear transformation of the input data that has achieved remarkable performance in wireless networks \cite{Eisen19-Wireless}, decentralized control of robot swarms \cite{Tolstaya19-Flocking} and recommendation systems \cite{Ruiz19-Nonlinear}, among others \cite{Kipf17-ClassifGCN, Ruiz19-GraphRNN}. Graph filters and GNNs rely heavily on the knowledge of the GSO. But if we do not know the graph and need to estimate it \cite{Segarra17-Template}, or if the graph changes with time \cite{Gama19-Control}, or if we want to train on one graph but test on another (transfer learning) \cite{Gama19-Scattering}, then it is of utmost importance to characterize how graph filters and GNNs react to changes in the underlying graph support.

In this paper, we start by considering GSOs that are permutations of each other and prove that graph convolutions are unaffected by these node relablings (permutation equivariance). Then, we prove that for graph convolutions computed on two arbitrary GSOs, the output will differ in a manner proportional to the relative distance between the GSOs (stability to relative perturbations). These results show that there is a trade-off between stability and discriminability for linear graph filters, but that this trade-off can be overcome by the use of pointwise nonlinearities. This renders GNNs both stable and discriminative, a feat that cannot be achieved by linear graph filters.

Stability results in graph neural networks have only been developed, so far, for graph scattering transforms, which involve carefully designed, non-trainable filters \cite{ZouLerman18-Scattering, Gama19-ScatteringDiffusion, Gama19-Scattering}. In \cite{ZouLerman18-Scattering}, stability to permutations is studied, as well as to perturbations on the eigenvalues and eigenvectors of the underlying graph support. In \cite{Gama19-ScatteringDiffusion} graph perturbations are measured in terms of the diffusion distance, while in \cite{Gama19-Scattering} different graph wavelets are compared in their stability \cite{Hammond11-Wavelets, Shuman15-Wavelets}.

In Sec.~\ref{sec:graphFilters} we introduce the GSP framework, define graph convolutions, and prove stability for linear graph filters. In Sec.~\ref{sec:GNN} we prove stability for GNNs and discuss how the conditions imposed on filters determine a trade-off between discriminability and stability, which can be overcome by the inclusion of pointwise nonlinearities. Finally, in Sec.~\ref{sec:sims} we show how to train a GNN while controlling for its stability and run experiments on a movie recommendation problem. Conclusions are drawn in Sec.~\ref{sec:conclusions}.

\vspace{-0.15cm}


\section{Stability of Graph Filters} \label{sec:graphFilters}

Let $\ccalG=(\ccalV, \ccalE, \ccalW)$ be a graph with a set of $N$ nodes $\ccalV$, a set of edges $\ccalE \subseteq \ccalV \times \ccalV$ and an edge weight function $\ccalW: \ccalE \to \reals_{+}$. This graph acts as the underlying support for the available data $\bbx \in \reals^{N}$. That is, $\bbx$ is modeled as a \emph{graph signal} where each entry $[\bbx]_{n}=x_{n}$ corresponds to the data value assigned to node $n$ \cite{Sandryhaila13-DSPG, Shuman13-SPG}. The data $\bbx$ is related to the underlying graph support by means of a linear map between graph signals $\bbS: \reals^{N} \to \reals^{N}$ that we denote a \emph{graph shift operator} (GSO) \cite{Ortega18-GSP}. The GSO is a linear operator $\bbS$ that updates the data value on each node by a weighted average of the values at neighboring nodes, i.e. it \emph{shifts} the signal across the graph. Therefore, the GSO can be written as a $N \times N$ matrix that respects the sparsity of the graph, $[\bbS]_{ij} = s_{ij}= 0$ if $i \neq j$ and $(j,i) \notin \ccalE$, and the value of the output signal at node $i$ is
\begin{equation} \label{eqn:graphShift}
[\bbS \bbx]_{i} = \sum_{j=1}^{N} [\bbS]_{ij} [\bbx]_{j} = \sum_{j \in \ccalN_{i}} s_{ij} x_{j}
\end{equation}
where $\ccalN_{i} = \{j \in \ccalV : (j,i) \in \ccalE\}$ is the set of neighboring nodes of $i$, and the last equality follows from the sparsity pattern of matrix $\bbS$. Examples of GSO typically used in the literature include the adjacency matrix \cite{Sandryhaila13-DSPG}, the Laplacian matrix \cite{Shuman13-SPG} and the Markov matrix \cite{Heimowitz17-MarkovGSP}.

To process data $\bbx$ in a manner that takes into account the irregular structure imposed by the underlying graph we need operations built on \eqref{eqn:graphShift}. In this light, we define the \emph{graph convolution} as a linear combination of shifted versions of the signal \cite{Segarra17-Linear}
\begin{equation} \label{eqn:graphConv}
\bby = \sum_{k=0}^{K-1} h_{k} \bbS^{k} \bbx = \bbH(\bbS) \bbx
\end{equation}
where $\bbh = [h_{0},\ldots,h_{K-1}] \in \reals^{K}$ is a set of $K$ filter taps, each one weighing the information located at the $k$-hop neighborhood. We say that $\bbH(\bbS)$ is a \emph{graph filter} \cite{Segarra17-Linear}. We note that \eqref{eqn:graphConv} can be calculated by means of $K-1$ exchanges of information with the one-hop neighborhood. Also, \eqref{eqn:graphConv} boils down to regular convolution when modeling time signals as being supported on a directed cycle, see \cite{Gama19-Architectures} for details.

Graph filters exhibit the key property of \emph{permutation equivariance}. Define $\ccalP = \{\bbP \in \{0,1\}^{N \times N}: \bbP \bbone = \bbone, \bbP^{\Tr} \bbone = \bbone\}$ the set of all permutation matrices. We prove the following.
%
\begin{theorem}[Permutation equivariance]
    \label{thm:permutationEquivariance}
Let $\bbS$ be the GSO of a graph $\ccalG$ and $\hbS = \bbP^{\Tr} \bbS \bbP$ be the GSO of a permuted version of $\ccalG$. Likewise, consider signals $\bbx$ and $\hbx = \bbP^{\Tr} \bbx$. Then,
\begin{equation}
    \bbH(\hbS) \hbx = \bbP^{\Tr} \bbH(\bbS) \bbx.
\end{equation}
\end{theorem}
\begin{proof}
    See \cite[Appendix A]{Gama19-Stability}.
\end{proof}
\noindent Theorem~\ref{thm:permutationEquivariance} essentially states that a graph filter on a permuted graph, applied to a correspondingly permuted signal, yields an output that is a permuted version of the original output. The immediate consequence of this theorem, is that graph filters are invariant to node relabelings.

We are ultimately interested in the effect that a more general perturbation of the GSO has on the output of a graph filter with fixed filter taps. Given two GSOs $\bbS$ and $\hbS$, consider the distance between filters $\bbH(\bbS)$ and $\bbH(\hbS)$ to be
\begin{equation} \label{eqn:filterDistance}
    \|\bbH(\bbS) - \bbH(\hbS)\|_{\ccalP} = \min_{\bbP \in \ccalP} \max_{\bbx: \|\bbx\|=1} \| \bbP^{\Tr} \bbH(\bbS) \bbx - \bbH(\bbP^{\Tr} \hbS \bbP)  \bbP^{\Tr} \bbx \|.
\end{equation}
We note that \eqref{eqn:filterDistance} is the operator norm, modulo permutations. We also note that if $\hbS = \bbP^{\Tr} \bbS \bbP$, then $\|\bbH(\bbS) - \bbH(\hbS)\|_{\ccalP} = 0$ in virtue of Thm.~\ref{thm:permutationEquivariance}. To measure the size of the GSO perturbation, we denote by $\ccalE(\bbS, \hbS) = \{\bbE: \bbP^{\Tr} \bbS \bbP = \bbS + (\bbE \bbS+ \bbS \bbE) , \bbP \in \ccalP\}$ the set of relative errror matrix and define the distance between $\bbS$ and $\hbS$ as
\begin{equation} \label{eqn:GSOdistance}
    d(\bbS, \hbS) = \min_{\bbE \in \ccalE(\bbS, \hbS)} \| \bbE\|.
\end{equation}
In essence, given a set of filter taps $\bbh$, we want to determine how much $\|\bbH(\bbS) - \bbH(\hbS)\|_{\ccalP}$ changes in relation to $d(\bbS, \hbS)$.

\def \thisplotscale {2.3}
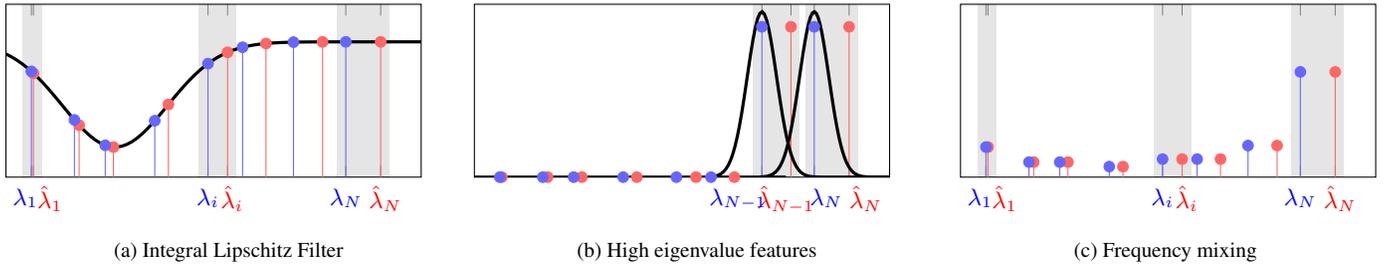
\begin{figure*}[t]
    \begin{subfigure}{0.33\textwidth}

\def \unit {\thisplotscale cm}

\def \frequencyresponse 
     { 0.9 - 0.7*exp(-(0.7*(x-1.6))^2) }

\hspace{-2.9mm}
\begin{tikzpicture}[x = 1*\unit, y=1*\unit]

\def \factorx {2.4/8}
\def \deltax  {0.5*\factorx}
\def \shadeshift  {0.05}

\path [fill=black!20, opacity = 0.5] 
              (\deltax - 0.001*\factorx - \shadeshift, 0.00) rectangle 
              (\deltax + 0.030*\factorx + \shadeshift, 1.00);
\path [fill=black!20, opacity = 0.5] 
              (\deltax + 3.393*\factorx - \shadeshift, 0.00) rectangle 
              (\deltax + 3.770*\factorx + \shadeshift, 1.00);
\path [fill=black!20, opacity = 0.5] 
              (\deltax + 6.048*\factorx - \shadeshift, 0.00) rectangle 
              (\deltax + 6.720*\factorx + \shadeshift, 1.00);

\begin{axis}[scale only axis,
             width  = 2.4*\unit,
             height = 1*\unit,
             xmin = -0.5, xmax=7.5,
             xtick = {0.03, -0.01, 3.77, 3.393, 6.72, 6.048},
             xticklabels = {\red{$\qquad\hat{\lambda}_1\phantom{\lambda}$},
                            \blue{$\lambda_1\ \ $}, 
                            \red{$\quad\hat{\lambda}_i\phantom{\lambda}$}, 
                            \blue{$\lambda_i$},
                            \red{$\quad\hat{\lambda}_{N}\phantom{\lambda}$},
                            \blue{$\lambda_N$}},
             ymin = -0, ymax = 1.15,
             ytick = {-1},
             typeset ticklabels with strut,
             enlarge x limits=false]

\addplot+[samples at = {0.03, 0.91, 1.57, 
                        2.63, 3.77, 4.51, 
                        5.60, 6.72}, 
          color = red!60, 
          ycomb, 
          mark=otimes*, 
          mark options={red!60}]
         {\frequencyresponse};

\addplot+[samples at = {-0.01, 0.819, 1.413, 
                        2.367, 3.393, 4.059, 
                        5.04, 6.048}, 
          color = blue!60, 
          ycomb, 
          mark=oplus*, 
          mark options={blue!60}]
         {\frequencyresponse};

\addplot[ domain=-0.5:7.5, 
          samples = 80, 
          color = black,
          line width = 1.2]
         {\frequencyresponse};

\end{axis}
\end{tikzpicture}

        \caption{Integral Lipschitz Filter}
        \label{subfig:ILfilter}
    \end{subfigure}
    \hfill
    \begin{subfigure}{0.33\textwidth}
        \centering

\def \unit {\thisplotscale cm}

\def \frequencyresponse 
     {1.1*exp(-(2.5*(x-6.048))^2}
\def \frequencyresponsetwo 
     {1.1*exp(-(2.5*(x-5.04))^2}

\hspace{-2.9mm}
\begin{tikzpicture}[x = 1*\unit, y=1*\unit]

\def \factorx {2.4/8}
\def \deltax  {0.5*\factorx}
\def \shadeshift  {0.05}

\path [fill=black!20, opacity = 0.5] 
              (\deltax + 6.048*\factorx - \shadeshift, 0.00) rectangle 
              (\deltax + 6.720*\factorx + \shadeshift, 1.00);

\path [fill=black!20, opacity = 0.5] 
              (\deltax + 5.04*\factorx - \shadeshift, 0.00) rectangle 
              (\deltax + 5.60*\factorx + \shadeshift, 1.00);

\begin{axis}[scale only axis,
             width  = 2.4*\unit,
             height = 1*\unit,
             xmin = -0.5, xmax=7.5,
             xtick = {5.60, 5.04, 6.72, 6.048},
             xticklabels = {\red{$\qquad\hat{\lambda}_{N-1}\phantom{\lambda_{N-1}}$},
                            \blue{$\lambda_{N-1}\qquad  $}, 
                            \red{$\qquad\hat{\lambda}_{N}\phantom{\lambda}$},
                            \blue{$\quad\lambda_N$}},
             ymin = -0, ymax = 1.15,
             ytick = {-1},
             typeset ticklabels with strut,
             enlarge x limits=false]

\addplot+[samples at = {0.03, 0.91, 1.57, 
                        2.63, 3.77, 4.51}, 
          color = red!60, 
          ycomb, 
          mark=otimes*, 
          mark options={red!60}]
         {0};

\addplot+[samples at = {6.72, 5.60}, 
          color = red!60, 
          ycomb, 
          mark=otimes*, 
          mark options={red!60}]
         {1};

\addplot+[samples at = {-0.01, 0.819, 1.413, 
                        2.367, 3.393, 4.059}, 
          color = blue!60, 
          ycomb, 
          mark=oplus*, 
          mark options={blue!60}]
         {0};

\addplot+[samples at = {6.048, 5.04}, 
          color = blue!60, 
          ycomb, 
          mark=oplus*, 
          mark options={blue!60}]
         {1};

\addplot[ domain=-0.5:5.5, 
          samples = 2, 
          color = black,
          line width = 1.2]
         {0};

\addplot[ domain=5.0:7.5, 
          samples = 70, 
          color = black,
          line width = 1.2]
         {\frequencyresponse};

\addplot[ domain=4.0:7.0, 
          samples = 70, 
          color = black,
          line width = 1.2]
         {\frequencyresponsetwo};
\addplot[ domain=7.0:7.5, 
          samples = 2, 
          color = black,
          line width = 1.2]
         {0};

\end{axis}
\end{tikzpicture}

        \caption{High eigenvalue features}
        \label{subfig:highFeature}
    \end{subfigure}
    \hfill
    \begin{subfigure}{0.33\textwidth}
        \centering

\def \unit {\thisplotscale cm}

\def \frequencyresponse 
     {   0.8*exp(-(1*(x-1.2))^2) 
       + 0.7*exp(-(0.7*(x-4))^2) 
       + 0.8*exp(-(1.4*(x-6))^2) 
       + 0.1}

\hspace{-2.9mm}
\begin{tikzpicture}[x = 1*\unit, y=1*\unit]

\def \factorx {2.4/8}
\def \deltax  {0.5*\factorx}
\def \shadeshift  {0.05}

\path [fill=black!20, opacity = 0.5] 
              (\deltax - 0.001*\factorx - \shadeshift, 0.00) rectangle 
              (\deltax + 0.030*\factorx + \shadeshift, 1.00);
\path [fill=black!20, opacity = 0.5] 
              (\deltax + 3.393*\factorx - \shadeshift, 0.00) rectangle 
              (\deltax + 3.770*\factorx + \shadeshift, 1.00);
\path [fill=black!20, opacity = 0.5] 
              (\deltax + 6.048*\factorx - \shadeshift, 0.00) rectangle 
              (\deltax + 6.720*\factorx + \shadeshift, 1.00);

\begin{axis}[scale only axis,
             width  = 2.4*\unit,
             height = 1*\unit,
             xmin = -0.5, xmax=7.5,
             xtick = {0.03, -0.01, 3.77, 3.393, 6.72, 6.048},
             xticklabels = {\red{$\qquad\hat{\lambda}_1\phantom{\lambda}$},
                            \blue{$\lambda_1\ \ $}, 
                            \red{$\quad\hat{\lambda}_i\phantom{\lambda}$}, 
                            \blue{$\lambda_i$},
                            \red{$\quad\hat{\lambda}_{N}\phantom{\lambda}$},
                            \blue{$\lambda_N$}},
             ymin = -0, ymax = 1.15,
             ytick = {-1},
             typeset ticklabels with strut,
             enlarge x limits=false]

\addplot+[color = red!60, 
          ycomb, 
          mark=otimes*, 
          mark options={red!60}]
          coordinates { (0.03, 0.20)
                        (0.91, 0.10)
                        (1.57, 0.10)
                        (2.63, 0.07)
                        (3.77, 0.12)
                        (4.51, 0.12)
                        (5.60, 0.21)
                        (6.72, 0.70)};

\addplot+[samples at = {-0.01, 0.819, 1.413, 
                        2.367, 3.393, 4.059, 
                        5.04, 6.048}, 
          color = blue!60, 
          ycomb, 
          mark=oplus*, 
          mark options={blue!60}]
          coordinates { (-0.010, 0.20)
                        ( 0.819, 0.10)
                        ( 1.413, 0.10)
                        ( 2.367, 0.07)
                        ( 3.393, 0.12)
                        ( 4.059, 0.12)
                        ( 5.040, 0.21)
                        ( 6.048, 0.70)};

\end{axis}
\end{tikzpicture}
        \caption{Frequency mixing}
        \label{subfig:frequencyMixing}
    \end{subfigure}
    \caption{\subref{subfig:ILfilter} Frequency response for an integral Lipschitz filter (in black), eigenvalues for $\bbS$ (in blue) and eigenvalues for $\hbS$ (in red). Note that larger eigenvalues exhibit a larger change. \subref{subfig:highFeature} Separating energy located at $\lambda_{N-1}$ from that at $\lambda_{N}$ requires filters with sharp transitions that are not integral Lipschitz. Then, a change in eigenvalues renders these filters useless (they are not stable) \subref{subfig:frequencyMixing} Applying a ReLU to a signal with all its energy located at $\lambda_{N}$ results in a signal with energy spread through the spectrum. Information on low eigenvalues can be discriminated in a stable fashion.\vspace{-0.2cm}}
    \label{fig:freqResponses}
\end{figure*}
%

We can use a frequency analysis to separate the action of any given filter on a signal into the effects of the specific filter taps and that of the underlying graph support. Let $\bbS = \bbV \bbLambda \bbV^{\Hr}$ be the eigendecomposition of the GSO, with $\bbV$ the eigenvector basis $\{\bbv_{n}\}_{n=1}^{N}$ and $\bbLambda$ a diagonal matrix containing the eigenvalues $\{\lambda_{n}\}_{n=1}^{N}$. The \emph{graph Fourier transform} (GFT) $\tbx$ of a signal $\bbx$ is computed as its projection onto the eigenvector basis of the GSO, $\tbx = \bbV^{\Hr} \bbx$ \cite{Sandryhaila14-Freq}. Then, the GFT of the output of a graph filter $\bby = \bbH(\bbS) \bbx$ becomes
\begin{equation} \label{eqn:filterOutputFrequency}
    \tby = \bbV^{\Hr} \big( \bbH(\bbS) \bbx\big) = \sum_{k=0}^{\infty} h_{k} \bbLambda^{k} \tbx = h(\bbLambda) \tbx.
\end{equation}
where the function $h: \reals \to \reals$ is called the filter's \emph{frequency response}
\begin{equation} \label{eqn:frequencyResponse}
    h(\lambda) = \sum_{k=0}^{\infty} h_{k} \lambda^{k}.
\end{equation}
We note that since $h$ is an analytic function, its application to a matrix is well defined. From \eqref{eqn:filterOutputFrequency} we see that the effect of the filter on the $i$th frequency coefficient is given by  $\tdy_{i} = h(\lambda_{i}) \tdx_{i}$. That is, the frequency content of $\bbx$ at the $i$th eigenvalue, gets modified by $h(\lambda_{i})$. This depends, on one hand, on the specific filter taps that determine the frequency response $h(\lambda)$ and, on the other hand, on the specific GSO under consideration that instantiates the frequency response on its eigenvalue $\lambda_{i}$. We thus note that the frequency response \eqref{eqn:frequencyResponse} is independent of the specific graph support, and only depends on the filter taps $\bbh$.

We can control the impact of changes in the GSO by carefully designing the frequency response \eqref{eqn:frequencyResponse}. In particular, we focus on filters that are \emph{integral Lipschitz}, see Fig.~\ref{subfig:ILfilter}.
%
\begin{definition}[Integral Lipschitz filters]
    \label{def:integralLipschitz}
Given filter taps $\bbh$ and frequency response $h(\lambda)$ [cf. \eqref{eqn:frequencyResponse}], we say that the corresponding filter is \emph{integral Lipschitz} if it satisfies that $|h(\lambda)| \leq 1$ and there exists a constant $C>0$ such that for all $\lambda_{1},\lambda_{2}$ it holds that
\begin{equation} \label{eqn:integralLipschitz}
    \big| h(\lambda_{2}) - h(\lambda_{1}) \big| \leq C \frac{| \lambda_{2} - \lambda_{1}|}{|\lambda_{1}+\lambda_{2}|/2}.
\end{equation}
\end{definition}
\noindent Integral Lipschitz filters are those whose frequency response is Lipschitz with a constant that depends on the midpoint value of the interval. Alternatively, see that filters that satisfy \eqref{eqn:integralLipschitz} also satisfy $|\lambda h'(\lambda)| \leq C$, where $h'$ is the derivative of $h$. These are filters that can vary arbitrarily fast for $\lambda \approx 0$, but have to be constant for $\lambda \to \infty$. We also note that this condition is reminiscent of the scale invariance of wavelet transforms \cite[Ch. 7]{Daubechies92-Wavelets} \cite{Hammond11-Wavelets, Shuman15-Wavelets}.

For filters that are integral Lipschitz, we can prove the following stability result.
%
\begin{theorem}[Stability of graph filters]
    \label{thm:graphFilterStability}
Let $\bbS$ and $\hbS$ be two GSOs such that $d(\bbS,\hbS) \leq \varepsilon$ where the error matrix $\bbE \in \ccalE(\bbS,\hbS)$ has an eigendecomposition $\bbE = \bbU \bbM \bbU^{\Hr}$. Consider a given set of filter taps $\bbh$ of an integral Lipschitz filter with constant $C$. Then, 
\begin{equation} \label{eqn:graphFilterStability}
    \big\| \bbH(\bbS) - \bbH(\hbS) \big\|_{\ccalP} \leq 2C \left( 1+ \delta \sqrt{N} \right) \varepsilon + \ccalO(\varepsilon^{2})
\end{equation}
with $\delta := (\|\bbU - \bbV+1)^{2}-1$ measuring the eigenvector basis misalingment.
\end{theorem}
\begin{proof}
    See \cite[Appendix C]{Gama19-Stability}.
\end{proof}
\noindent Theorem~\ref{thm:graphFilterStability} shows that the change at the output of a graph filter due to changes in the underlying GSO is proportional to the distance between those GSOs [cf. \eqref{eqn:GSOdistance}]. The proportionality constant can be analyzed in two separate parts. First, we have the integral Lipschitz constant $C$ which can be controlled by careful design of the frequency response (filter taps). Second, $(1+\delta \sqrt{N})$ depends on the specific family of perturbations and worsens as the graph grows larger. This second part cannot be controlled and is dependent on the specific perturbations the graph support will be subject to. However, we can impose a structural constraint on the perturbation matrix $\bbE$ to obtain a constant that only depends on $C$. See \cite[Thm.~4]{Gama19-Stability} for details.


\section{Stability of Graph Neural Networks} \label{sec:GNN}

A graph neural network (GNN) is a nonlinear map $\bbPhi(\bbS, \bbx)$ that is applied to the input $\bbx$ and takes into account the underlying graph by means of the GSO $\bbS$. It consists of a cascade of $L$ layers, each of them applying a graph filter $\bbH_{\ell}(\bbS)$ followed by a pointwise nonlinearity $\sigma_{\ell}$ (activation function)
\begin{equation} \label{eqn:GNN}
    \bbx_{\ell} = \sigma_{\ell} \big( \bbH_{\ell}(\bbS) \bbx_{\ell-1} \big)
\end{equation}
for $\ell=1,\ldots,L$, where $\bbx_{0}=\bbx$ the input signal, and $\bbPhi(\bbS,\bbx) = \bbx_{L}$ the output of the last layer \cite{Bruna14-DeepSpectralNetworks, Defferrard17-CNNGraphs, Gama19-Architectures}.

The use of graph filters \eqref{eqn:graphConv} as $\bbH_{\ell}(\bbS)$ means that GNNs inherit the properties of permutation equivariance and stability to relative perturbations from them.
%
\begin{theorem}[Permutation equivariance of GNNs]
    \label{thm:permutationEquivarianceGNN}
Let $\bbS$ be the GSO of a graph $\ccalG$ and $\hbS = \bbP^{\Tr} \bbS \bbP$ be the GSO of a permuted version of $\ccalG$. Likewise, consider signals $\bbx$ and $\hbx = \bbP^{\Tr} \bbx$. Then,
\begin{equation} \label{eqn:permutationEquivarianceGNN}
    \bbPhi(\hbS, \hbx) = \bbP^{\Tr} \bbPhi(\bbS, \bbx).
\end{equation}
\end{theorem}
\begin{proof}
    See \cite[Appendix D]{Gama19-Stability}.
\end{proof}
\noindent Result in Thm.~\ref{thm:permutationEquivarianceGNN} follows through because the nonlinearities are pointwise and therefore applied separately to each node, bearing no effect on the permutation equivariance from graph filters. We note that there are local activation functions involving neighboring exchanges that also preserve the permutation equivariance property \cite{Ruiz19-Nonlinear}.

For the stability result to hold, we need to use pointwise nonlinearities that are \emph{normalized} Lipschitz, i.e. Lipschitz functions with constant equal to $1$, $|\sigma_{\ell}(b) - \sigma_{\ell}(a)| \leq |b - a|$ for all $b,a \in \reals$, and for all $\ell$. We note that typically used activation functions like ReLU or $\tanh$ satisfy this condition.
%
\begin{theorem}[Stability of GNNs]
    \label{thm:GNNstability}
Let $\bbS$ and $\hbS$ be two GSOs such that $d(\bbS,\hbS) \leq \varepsilon$ where the error matrix $\bbE \in \ccalE(\bbS,\hbS)$ has an eigendecomposition $\bbE = \bbU \bbM \bbU^{\Hr}$. Consider a GNN $\bbPhi$ with $L$ layers where, in each layer, $\sigma_{\ell}$ is Lipschitz with constant $1$ and the filter $\bbh_{\ell}$ is integral Lipschitz with constant $C_{\ell}$. Then,
\begin{equation} \label{eqn:GNNstability}
    \big\| \bbPhi(\bbS, \cdot) - \bbPhi(\hbS, \cdot) \big\|_{\ccalP} \leq 2C \left( 1+ \delta \sqrt{N} \right) L \varepsilon + \ccalO(\varepsilon^{2})
\end{equation}
with $C = \max_{\ell} \{C_{\ell}\}$ and $\delta := (\|\bbU - \bbV+1)^{2}-1$ measuring the eigenvector basis misalingment.
\end{theorem}
\begin{proof}
    See \cite[Appendix E]{Gama19-Stability}.
\end{proof}
\noindent Thm.~\ref{thm:GNNstability} proves that GNNs are stable in the sense that, if the constitutive filters are integral Lipschitz and the activation function is normalized Lipschitz, then a change of $\varepsilon$ in the GSOs causes a change proportional to $\varepsilon$ in the output of the GNNs. The proportionality constant has the term $C$ that depends on the filter design, and the term $(1+\delta\sqrt{N})$ that depends on the specific perturbation under consideration. But it also has a constant factor $L$ that depends on the depth of the architecture. Therefore, the deeper a GNN is the less stable it becomes. This is due to how the errors propagate and amplify through subsequent applications of graph filters.

We have proven that graph filters have the properties of permutation equivariance and stability to relative graph perturbations, and that GNNs inherit these properties. We have also observed that the corresponding graph filters have to be integral Lipschitz for stability to hold, and that the integral Lipschitz constant controls the level of stability. This observation helps explain why GNNs exhibit better performance when dealing with signals with relevant high-eigenvalue frequency content. To see this, consider the following example.

Let $\bbS$ be the GSO of a given graph, and consider the perturbation $\hbS$ to be an edge dilation
\begin{equation}
    \hbS = (1+\varepsilon) \bbS
\end{equation}
where all edges are increased proportionally by a factor of $\varepsilon$. Clearly, $\bbE = (\varepsilon/2) \bbI$ and $d(\bbS, \hbS) = \|\bbE\| \leq \varepsilon$. The eigenvalues are now $\hat{\lambda}_{n} = (1+\varepsilon) \lambda_{n}$ while the eigenvectors remain the same. We note that, even if $\varepsilon$ is very small, the change in eigenvalues could be large if $\lambda_{n}$ is large, see Fig.~\ref{subfig:ILfilter}.

To account for this variability in large eigenvalues (even for small $\varepsilon$) we need to design the frequency response [cf. \eqref{eqn:frequencyResponse}] so as to absorb these changes. Otherwise, the output of filtering [cf. \eqref{eqn:filterOutputFrequency}] could change significantly (the values of $h(\lambda_{n})$ and $h(\hat{\lambda}_{n}$) could be very different), making it unstable. The integral Lipschitz condition on filters, precisely avoids this problem by forcing the frequency response to be flat for large eigenvalues, see Fig.~\ref{subfig:ILfilter}.

\begin{figure*}[t]
    \centering
    \begin{subfigure}{0.3\textwidth}
        \centering
        \includegraphics[width=\textwidth]{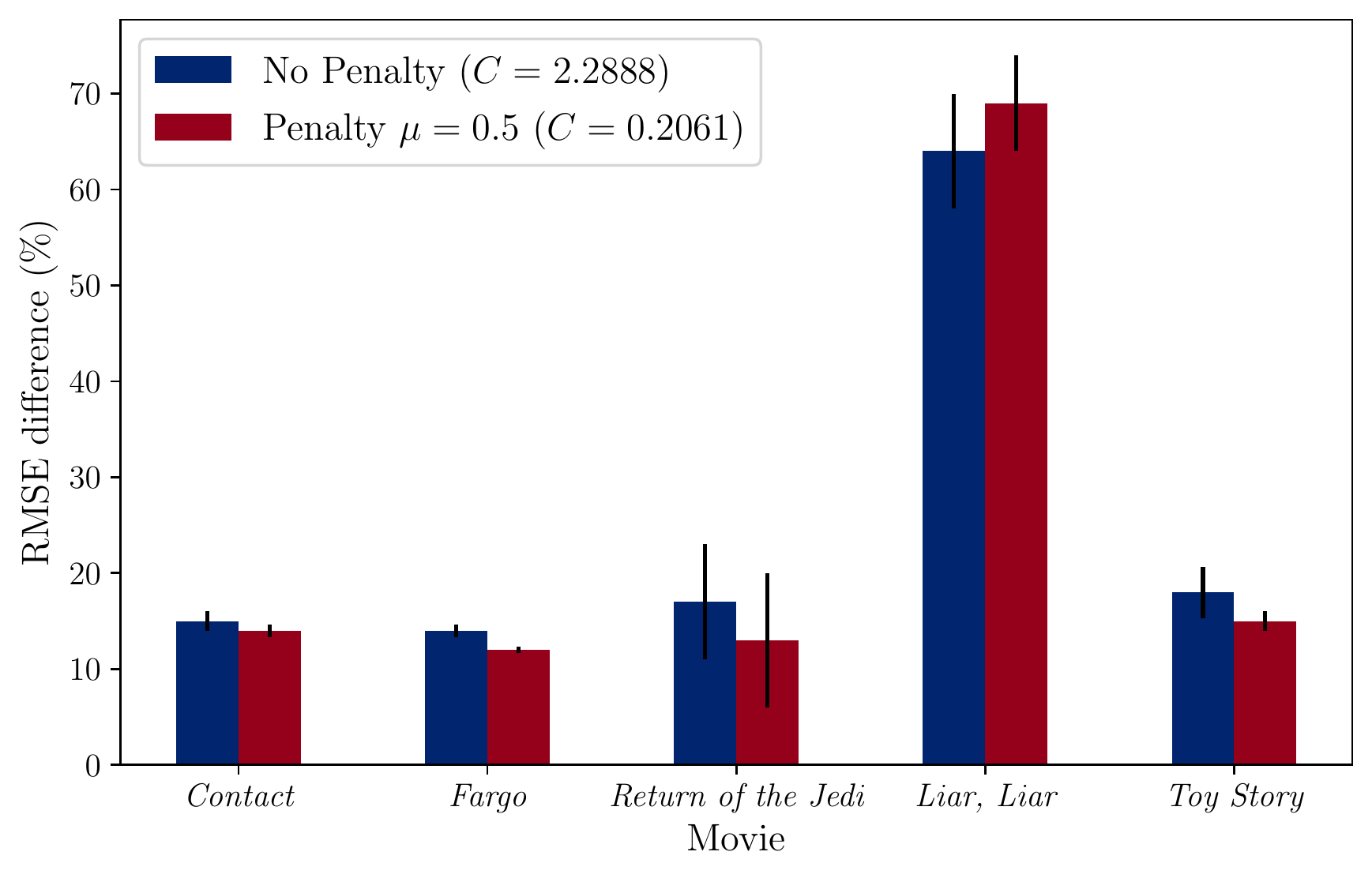}
        \caption{Train on 'Star Wars', test on other movies}
        \label{subfig:relativeRMSE}
    \end{subfigure}
    \hfill
    \begin{subfigure}{0.3\textwidth}
        \centering
        \includegraphics[width=\textwidth]{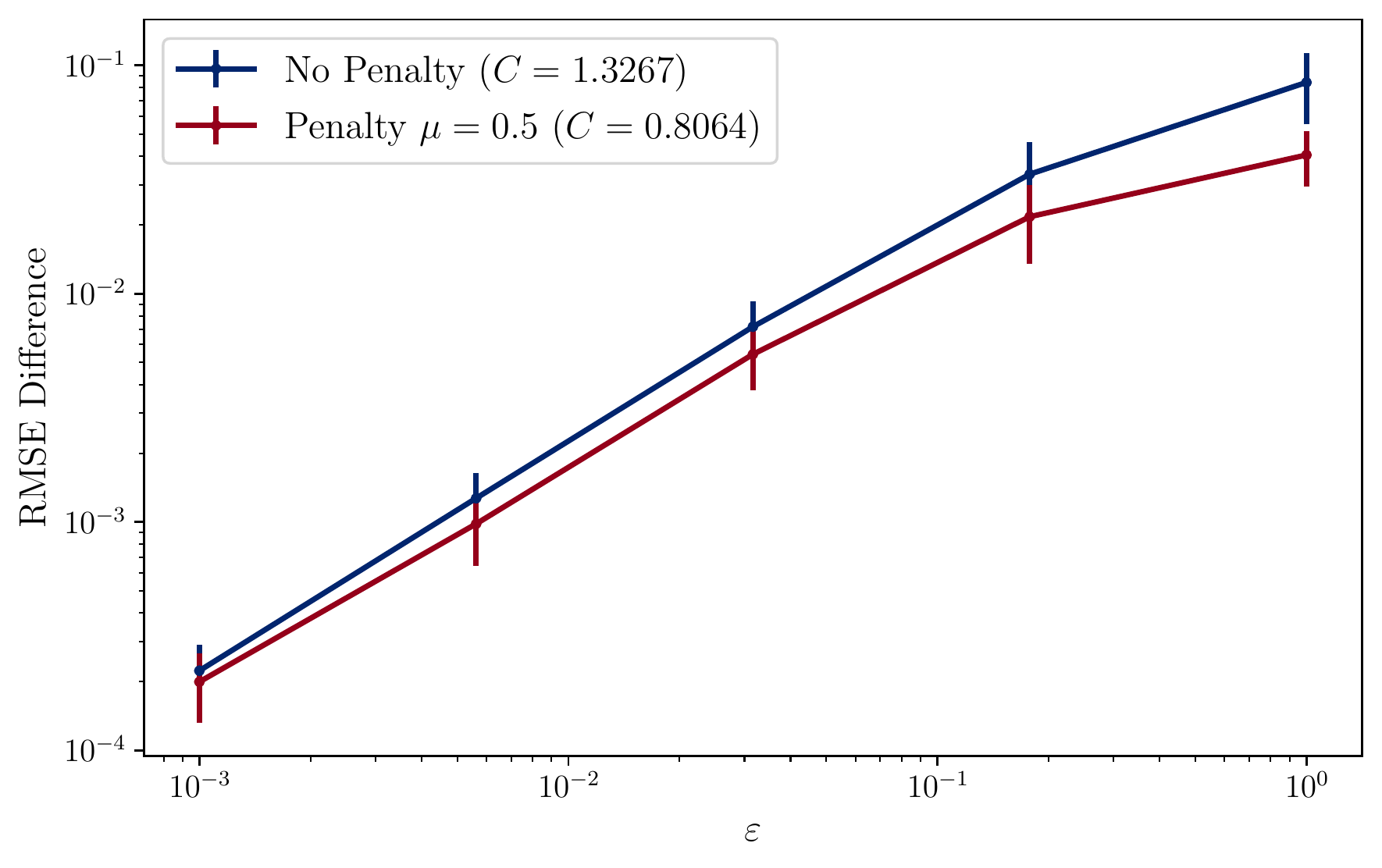}
        \caption{Synthetic perturbations}
        \label{subfig:graphVar}
    \end{subfigure}
    \hfill
    \begin{subfigure}{0.3\textwidth}
        \centering
        \includegraphics[width=\textwidth]{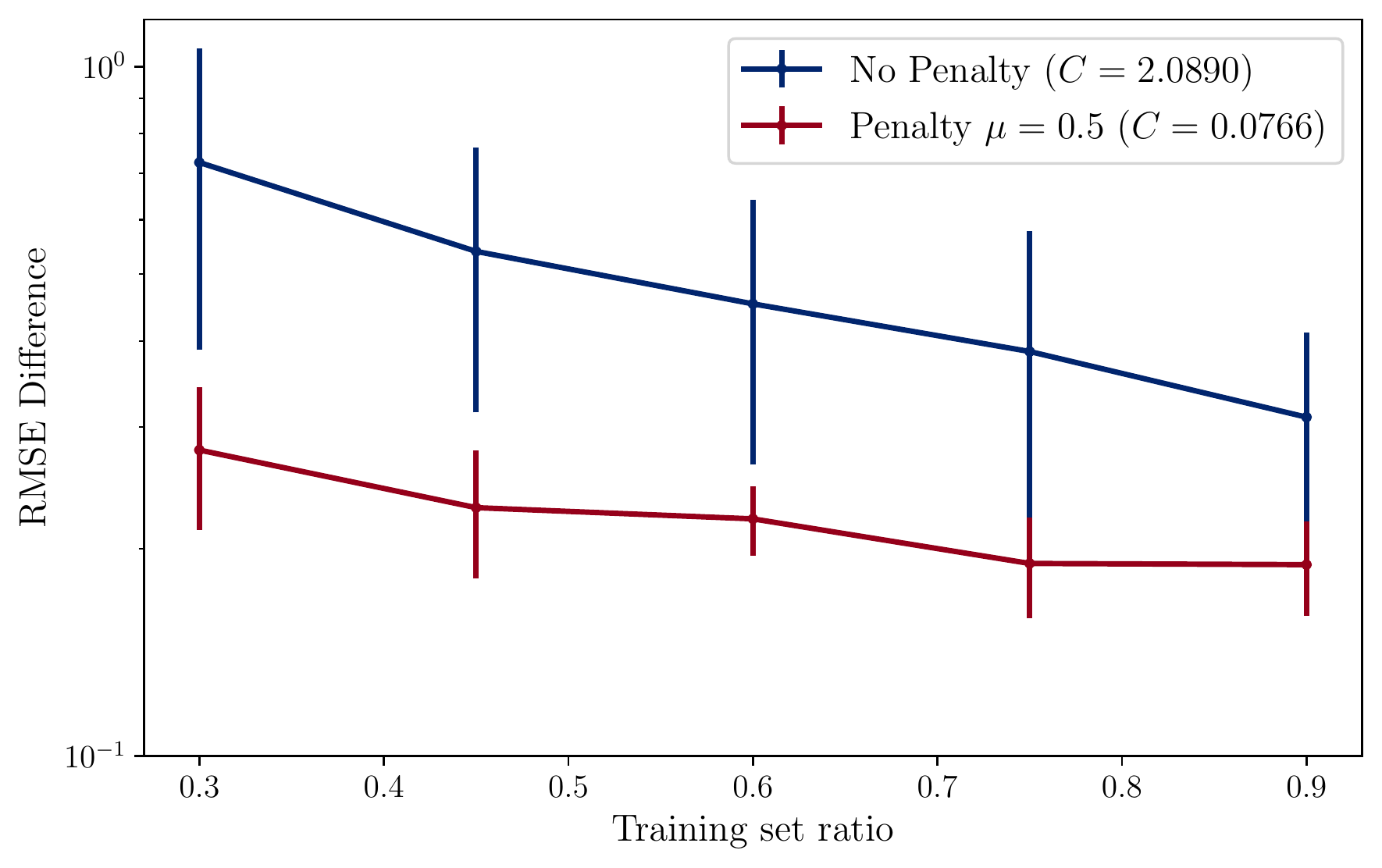}
        \caption{Change in train/test ratio}
        \label{subfig:splitVar}
    \end{subfigure}
    \caption{\subref{subfig:relativeRMSE} RMSE degradation (in percentage) when the architectures are trained to predict the rating of the movie \emph{Star Wars}, but are tested on $5$ other movies. We see that, except for the case of \emph{Liar, Liar}, the RMSE degradation is below $20\%$ with the more stable architecture having a degradation below $15\%$. \subref{subfig:graphVar} RMSE degradation when testing on GSOs $\hbS$ that have been synthetically perturbed within a relative distance of $\varepsilon$. The more stable architecture exhibits a smaller RMSE degradation. \subref{subfig:splitVar} RMSE degradation when testing the architectures on different GSOs obtained by changing the training/test ratio. When we have a ratio similar to that on which the architecture was trained, the RMSE degradation is lower.\vspace{-0.2cm}}
    \label{fig:sims}
\end{figure*}
%

The cost to pay for stability, however, is that integral Lipschitz filters are not able to discriminate information located at higher eigenvalues (Fig.~\ref{subfig:highFeature}). In essence, linear filters are either stable or discriminative, but cannot be both. GNNs on the other hand, incorporate pointwise nonlinearities at the output of the linear graph filter. This nonlinear operation has a frequency mixing effect by which the energy of the signal is spilled throughout the spectrum, see Fig.~\ref{subfig:frequencyMixing}. Then, the energy that appears in smaller eigenvalues can be arbitrarily discriminated, providing GNNs a way to stably discriminate signals with large eigenvalue content. Thus, GNNs are information processing architectures that are both stable and selective.

Finally, we remark that the above proofs and analysis can be readily extended to multi-feature signals (graph signal tensors) which assign a vector of $F$ features to each node, instead of a single scalar. Please, refer to \cite{Gama19-Stability} for details on the multi-feature setting.


\section{Numerical Experiments} \label{sec:sims}

Consider a given dataset of input-output pairs $\ccalT = \{(\bbx,\bby)\}$ where $\bbx$ is a graph signal defined on a graph with GSO $\bbS$. We want to use a GNN as a nonlinear map between the input $\bbx$ and the output $\bby$. We can thus use the given dataset to \emph{fit} or \emph{train} the neural network by finding the filter taps $\bbh_{\ell}$ that minimize some loss function $\min_{\bbh_{\ell}} \sum_{(\bbx,\bby) \in \ccalT} \ccalL[\bbPhi(\bbS, \bbx), \bby]$. We note that, in the context of training, the permutation equivariance property of GNNs serves as a form of \emph{data augmentation}. More precisely, by exploiting the topological symmetries of the underlying graph, the GNN can learn how to process the signal on all those parts of the graph that are topologically symmetric by seeing a sample in only one of them.

By minimizing the loss function, we obtain a given set of filter taps that are a good fit for the data at hand. However, the resulting frequency responses might not have a good stability constant. To overcome this, we add a penalty to the loss function in order to control the value of the stability constant
\begin{equation} \label{eqn:lossFunctionWithPenalty}
    \{\bbh_{\ell}\} = \argmin_{\bbh_{\ell}}  \sum_{(\bbx,\bby)\in\ccalT} \ccalL\big[ \bbPhi(\bbS,\bbx), \bby\big] + \mu \max_{\lambda \in [\lambda_{a},\lambda_{b}]} |\lambda h'(\lambda)|.
\end{equation}
We note that the bounds of the interval $[\lambda_{a},\lambda_{b}]$ have to be set before training begins. One option is to set it to the eigenvalue interval of the given GSO (this would demand an eigendecomposition, albeit only once before training begins). Alternatively, we can exploit well-known bounds relating the eigenvalues with the topology of the graph \cite{Cvetkovic79-SpectraGraphs, Godsil01-AlgebraicGraphTheory}. Furthermore, we note that the computation of the derivative $h'(\lambda)$ is straightforward, since it is also a polynomial with the same filter taps $\bbh_{\ell}$ that we are optimizing over. Finally, we remark that, by tweaking the penalty value $\mu$ we adjust the trade-off between better performance and more stability. If $\mu$ is too big, then the training would tend to set all filter taps to $0$, which is the trivial but perfectly stable solution.

In what follows, we consider the problem of movie recommendation systems \cite{Huang18-RatingGSP}. We are given a dataset of user ratings for some movies, and we want to \emph{learn} how a user would rate a specific movie given their previous ratings and all other users in the dataset. To do this, we build a graph where each node is a movie, and the edges are based on the Pearson correlation coefficient obtained from the pool of users that have rated any given pair of movies. See \cite[Sec.~II]{Huang18-RatingGSP} for details on the construction of this graph. We then prune this graph keeping only the $10$ nearest neighbors, we make it undirected by keeping the average edge weight and we adopt the resulting adjacency matrix as the GSO $\bbS$. Additionally, we model each user in the dataset as a graph signal $\bbx$, where the value $[\bbx]_{n} = x_{n}$ at node $n$ is the rating that the user has given to movie $n$. The ratings are integers between $1$ and $5$, and if no rating was given, then we assign $x_{n} = 0$ to said node.

We consider the MovieLens-100k \cite{Harper16-MovieLens} dataset, that consists of $100,000$ ratings given by $943$ users to $1,582$ movies. Following the above described model, this implies a dataset of $943$ graph signals defined over a graph with $1,582$ nodes. We focus on \emph{learning} the rating that any given user would give to a specific movie (node $n$) based on the ratings given to other movies (the graph signal $\bbx$) and the relationship with other users that have similar taste (given by the graph support $\bbS$). We consider all users that have rated the specific movie (have nonzero value $[\bbx]_{n} = x_{n} > 0$), take the rating $x_{n}$ as the label $y$ associated to the signal $\bbx$ and then zero-out the $n$th entry $[\bbx]_{n} = 0$ (to render the rating unknown). Specifically, we choose to estimate the rating at the movie \emph{Star Wars} given that is the one with the largest number of ratings. We use $90\%$ of the resulting dataset for training and $10\%$ for testing (no samples in the test set are included when estimating the graph).

The map between the graph signal $\bbx$ (ratings for some of the movies) and the target $y=x_{n}$ (rating for the specific movie) is parametrized by a single-layer GNN with $F_{1}=64$ output features, using graph filters with $K_{1}=5$ filter taps, followed by a ReLU nonlinearity. Since the output of this GNN $\bbx_{1}$ is another graph signal, we focus particularly on the value of the $64$ features at the node $n$ of interest. We further learn a readout layer consisting of a linear combination of the resulting $64$ features at node $n$ so that the final output is a single scalar predicting the rating given at said node. We note that all operations involved are local. First, a graph convolutional layer involving the application of the graph filter bank that demands $K_{1}-1 = 4$ exchanges with the one-hop neighborhood, and then a readout layer consisting of a linear transformation of the $64$ resulting features at the single node of interest (i.e. no involvement of values at any other node in this readout layer).

We train this GNN by minimizing the loss function \eqref{eqn:lossFunctionWithPenalty}, where $\ccalL$ is a smooth L1 loss. We consider two different training cases leading to two different models. The first one in which there is no penalty ($\mu=0$) and the second one where we set $\mu = 0.5$. We use the ADAM optimizer \cite{Kingma15-ADAM} with learning rate $0.005$ and forgetting factors $\beta_{1} = 0.9$ and $\beta_{2} = 0.999$. We train for $40$ epochs using batches of size $5$. In all subsequent experiments, we report averages over $10$ different dataset split realizations (the split is selected at random) and the corresponding standard deviation.

For the first experiment, we train the GNNs to estimate the rating of the movie \emph{Star Wars} both with no penalty ($\mu = 0$) and with stability penalty ($\mu = 0.5$). At test time, we obtain an RMSE of $0.8640$ $(\pm 0.1674)$ for the No Penalty GNN, and $0.8655$ $(\pm 0.1655)$ for the Penalty GNN. We then proceed to test these already trained GNNs on estimating the rating at some other movies, as shown on Fig.~\ref{subfig:relativeRMSE}. We see that, except for the case of the movie \emph{Liar, liar}, the RMSE degradation for estimating the rating of a movie the GNN was not trained for is below $20\%$. Moreover, in all cases, the more stable GNN (the one trained with the penalty) exhibits a degradation below $15\%$ and always better than the No Penalty GNN.

For the second experiment, we introduce a synthetic relative perturbation to the GSO $\bbS$ by randomly generating a GSO $\hbS$ such that $d(\bbS, \hbS) \leq \varepsilon$ [cf. \eqref{eqn:GSOdistance}]. We then test on $\hbS$ the GNNs trained on $\bbS$ for estimating the rating at node $n$ and compare the RMSE with that obtained when testing on $\bbS$. The results illustrated in Fig.~\ref{subfig:graphVar} show that the Penalty GNN is more stable than the No Penalty one, and that the gap in RMSE difference increases as the perturbation increases.

Finally, we consider perturbations arising from the randomness in choosing the training/test set split. Since we use the resulting training set to create $\bbS$, a change in the training set selection results in a different $\hbS$. In Fig.~\ref{subfig:splitVar} we show the RMSE difference between testing on the trained $0.9/0.1$ partition and testing on other partitions. Again, we observe that the Penalty GNN is more stable.

\vspace{-0.1cm}


\section{Conclusions} \label{sec:conclusions}

In this work we discussed the stability properties of graph filters and GNNs. We proved that both are permutation equivariant and are stable to relative perturbations of the underlying graph support. We observed that a condition for stability is that graph filters need to have a flat frequency response at large eigenvalues. We then argued that this prevents graph filters to be able to discriminate features located on these eigenvalues, and that this is a fundamental limitation of graph filters, which can thus be either stable or selective, but not both. GNNs, on the other hand, use the frequency mixing effect of nonlinearities to spread the information content throughout the eigenvalue spectrum, especially on lower eigenvalues where it can be separated in a stable fashion. Thus, GNNs are information processing architectures that are both discriminative and stable. The improved performance of GNNs over graph filters thus becomes more marked when processing signals where the relevant information content is in high frequencies. Finally, we run experiments on a movie recommendation problem, where we show that more stable architectures exhibit a better performance when tested on different conditions than the ones they were trained on.


\vfill\pagebreak

\bibliographystyle{IEEEbib}
\bibliography{myIEEEabrv,biblioStability}

\end{document}